\documentclass[preprint]{article}

\PassOptionsToPackage{numbers, compress}{natbib}
\usepackage{neurips_2019}

\usepackage[utf8]{inputenc} 
\usepackage[T1]{fontenc}    
\usepackage{hyperref}       
\usepackage{url}            
\usepackage{booktabs}       
\usepackage{amsfonts}       
\usepackage{dsfont}
\usepackage{nicefrac}       
\usepackage{microtype}      
\usepackage{todonotes}
\usepackage{amsmath,amssymb}
\usepackage{amsthm}
\usepackage{booktabs}
\usepackage{multirow}
\usepackage{graphicx}
\usepackage{algorithm}
\usepackage[noend]{algpseudocode}

\newtheorem{theorem}{Theorem}

\newcommand{\E}{\mathbb{E}}

\DeclareMathOperator*{\arginf}{arg\,inf}
\newcommand{\lpnorm}[2]{\left\lVert#1\right\rVert_{#2}}

\newcommand{\set}[2]{\left\{ #1 \mid #2 \right\}}
\newcommand{\divergence}[2]{D\!\left(#1 \middle\| #2 \right)}

\newcommand{\W}[2]{W\!\left(#1, #2 \right)}
\newcommand{\OP}[3]{W_{#3}\!\left(#1, #2 \right)}

\newcommand{\mean}[2]{\E_{#2} \! \left[ #1 \right]}

\newcommand{\de}[1]{\text{d}#1}
\newcommand{\restr}[2]{#1|_{#2}}

\title{k-GANs: Ensemble of Generative Models with Semi-Discrete Optimal Transport}

\author{
  Luca Ambrogioni \\
  Radboud University\\
  \texttt{l.ambrogioni@donders.ru.nl} \\
  \And
  Umut Güçlü \\
  Radboud University\\
  \texttt{u.guclu@donders.ru.nl} \
  \And
  Marcel A. J. van Gerven \\
  Radboud University\\
  \texttt{m.vangerven@donders.ru.nl} \\
}

\begin{document}

\maketitle
\begin{abstract}
Generative adversarial networks (GANs) are the state of the art in generative modeling. Unfortunately, most GAN methods are susceptible to mode collapse, meaning that they tend to capture only a subset of the modes of the true distribution. A possible way of dealing with this problem is to use an ensemble of GANs, where (ideally) each network models a single mode. In this paper, we introduce a principled method for training an ensemble of GANs using semi-discrete optimal transport theory. In our approach, each generative network models the transportation map between a point mass (Dirac measure) and the restriction of the data distribution on a tile of a Voronoi tessellation that is defined by the location of the point masses. We iteratively train the generative networks and the point masses until convergence. The resulting k-GANs algorithm has strong theoretical connection with the k-medoids algorithm. In our experiments, we show that our ensemble method consistently outperforms baseline GANs. 

\end{abstract}
\section{Introduction} 

Optimal transport theory is becoming an essential tool for modern machine learning research. The development of efficient optimal transport algorithms led to a wide range of machine learning applications \cite{peyre2017computational, cuturi2013sinkhorn, solomon2014wasserstein, kloeckner2015geometric, ho2017multilevel, arjovsky2017wasserstein, patrini2018sinkhorn, lee2018minimax, genevay18a, staib2017parallel, ambrogioni2018wasserstein, mi2018variational}. A notable example is approximate Bayesian inference, where optimal transport techniques have been used for constructing probabilistic autoencoders \cite{tolstikhin2017wasserstein, patrini2018sinkhorn} and for general purpose variational Bayesian inference \cite{ambrogioni2018wasserstein}. However, the field that has been most deeply influenced by optimal transport theory is arguably generative modeling \cite{arjovsky2017wasserstein, genevay18a, gulrajani2017improved, adler2018banach, genevay2017gan, gemici2018primal}. The introduction of the Wasserstein generative adversarial network (wGAN) \cite{arjovsky2017wasserstein} was a milestone as it provided a more stable form of adversarial training. Generative adversarial networks (GANs) greatly improved the state of the art in image generation. Nevertheless, GAN training often leads to mode collapse, where part of the data space is ignored by the generative model. A possible way to mitigate this phenomenon is to use a collection of GANs, each modeling a part of the data space \cite{wang2016ensembles}. However, existing ensembling techniques are often heuristic in nature and do not provide a principled way for ensuring that the different generators model non-overlapping parts of the data space. In this paper, we derive an ensemble of generative models algorithm from first principles using the theory of semi-discrete optimal transport. The basic idea is to jointly learn a series of elements (prototypes) of a discrete distribution and the optimal transportation functions from these prototypes to the data space. The notion of optimality is determined by a transportation cost that measures the dissimilarity between the prototypes and the data points.  The resulting k-GANs algorithm has strong theoretical connections with the k-means and k-medoids algorithms. In the k-GANs algorithm we learn $k$ prototypes that implicitly define a partitioning of the data space into non-overlapping cells. The distribution of the data in each of these cells is generated by a stochastic transportation function that maps the prototype into the data space. These transportation functions are parameterized by deep networks and are trained as regular GANs within their cell. The prototypes and the transportation functions are learned jointly so that the boundary of the cells shifts during training as a consequence of the changes in the prototypes.

\section{Related work}
From a theoretical point of view, our algorithm has a strong connection with the traditional k-means and k-medoid clustering methods \cite{forgy1965cluster, graf2007foundations}.
This connection between k-means and semi-discrete optimal transport stems from the fact that semi-discrete transport problems implicitly define a Laguerre tessellation of the space, which reduces to the more familiar Voronoi tessellation in special cases \cite{peyre2017computational, graf2007foundations}. Recently, this connection has been exploited in a variational clustering algorithm which uses optimal transport theory in order to derive a more powerful clustering method \cite{mi2018variational}.

\section{Background on optimal transport} 
In machine learning and statistics, optimal transport divergences are often used for comparing probability measures. Consider two probability measures $\nu(\de{x})$ and $\nu(\de{y})$. The optimal transport divergence between them is defined by the following optimization problem:
\begin{equation}
\OP{\mu}{\nu}{c} = \inf_{\gamma(\de{x},\de{y}) \in \Gamma} \int_{\mathcal{X} \times \mathcal{Y}} c(x,y) \gamma(\de{x},\de{y})~,  
\end{equation}
where $c(x,y)$ is the cost of transporting probability mass from $x$ to $y$ and $\Gamma$ is the set of probability measures that have $\nu(\de{x})$ and $\mu(\de{y})$ as marginal measures over $x$ and $y$ respectively. The transportation nature of this problem can be seen by slightly reformulating the objective function by writing the joint measure $\gamma(\de{x},\de{y})$ as the product of a conditional measure $\gamma(\de{x}|y)$ and of the marginal measure $\mu(\de{y})$:
\begin{equation}
\OP{\mu}{\nu}{c} = \inf_{\gamma(\de{x}|y) \in \Gamma_\nu} \int_{\mathcal{Y}} \left( \int_{\mathcal{X}} c(x,y) \gamma(\de{x}|y) \right) \mu(\de{y})~, 
\end{equation}
where the integral of the conditional measures $\gamma(\de{x}|y)$ has to be equal to $\nu(\de{x})$:
\begin{equation}
\Gamma_\nu = \set{\gamma(\de{x}|y)}{\int_\mathcal{Y} \gamma(\de{x}|y) \mu(\de{y}) = \nu(\de{x})}~. \label{eq: marginalization constraint}
\end{equation}
Therefore, the conditional measures $\gamma(\de{x}|y)$ can be interpreted as stochastic transportation functions that map each element of $y$ into a probability measure over $x$. 

\section{Semi-discrete optimal transport as ensemble of generative models} \label{sec: semi-discrete as ensemble}
One of the main advantages of using optimal transport divergences is that they can be defined between probability measures with very different support. An important case is semi-discrete transport, where $\nu(\de{x})$ is is absolutely continuous with respect to Lebesgue measure $\de{x}$ while $\mu(\de{y})$ is a Dirac measure:
$$
\mu(\de{y}) = \sum_j w_j \delta_{y_j}(\de{y})~.
$$
Semi-discrete optimal transport has important machine learning applications. For our purposes, the minimization of a semi-discrete optimal transport divergences can be used for approximating the probability distribution of the data with a discrete distribution over a finite number of ``prototypes''. The semi-discrete optimal transport divergence can be rewritten as follows:
\begin{align} \label{eq: semi-discrete optimal transport}
\OP{\nu}{\mu}{c} &= \inf_{\gamma(x|y)} \int_\mathcal{Y} \left( \int_\mathcal{X} c(x,y) \gamma(\de{x}|y) \right) \mu(\de{y})\\ 
&= \sum_j w_j \inf_{\gamma(x|y_j)} \int_\mathcal{X} c(x,y_j) \gamma(\de{x}|y_j)~, \label{eq: semi-discrete transport}
\end{align}
with the following constraint:
\begin{equation}\label{eq: marginalization constraint II}
\nu(\de{x}) = \sum_j  w_j \gamma(\de{x}|y_j)~.
\end{equation}
Note that each conditional measure $\gamma(\de{x}|y_j)$ can be interpreted as a generative model that maps a prototype $y_j$ into a probability distribution over the data points $x$. The optimization in Eq.~\ref{eq: semi-discrete transport} assures that these distributions are centered around their prototype (in a sense given by the cost function) while the marginalization constraint in Eq.~\ref{eq: marginalization constraint II} guarantees that the sum of all generative models is the real distribution of the data. In other words, the solution of this semi-discrete optimal transport problem provides an ensemble of local generative models. 

\section{Geometry of semi-discrete optimal transport problems}
In this section we will summarize some known results of semi-discrete optimal transport that will provide the theoretical foundation for our work. Semi-discrete optimal transport has a deep connection with geometry as it can be proven that the transportation maps are piecewise constant and define a tessellation of the space $\mathcal{X}$. In order to show this, it is useful to introduce the unconstrained dual formulation of the optimal transport problem \cite{peyre2017computational}:
\begin{equation}
\OP{\mu}{\nu}{c} = \sup_{\boldsymbol{g}} \int_\mathcal{X} g^c(x) \nu(\de{x}) + \sum_j g_j w_j~, \label{eq: dual optimization}
\end{equation}
where $g^c(x)$ denotes the c-transform of the vector of dual weights $\boldsymbol{g}$:
\begin{equation}
    g^c(x) = \inf_j \left(c(x,y_j) - g_j \right)~. \label{eq: c-transform}
\end{equation}
We can now reformulate the objective function in terms of a tessellation of $\mathcal{X}$:
\begin{equation}
\OP{\mu}{\nu}{c} = \sup_{\boldsymbol{g}} \mathcal{E}(\boldsymbol{g}) = \sup_{\boldsymbol{g}} \sum_j \int_{L_j(\boldsymbol{g})} \left( c(x,y_j) - g_j \right) \nu(\de{x}) + \sum_j g_j w_j~, \label{eq: dual optimization II}
\end{equation}
where the sets $L_j(g_j)$ are defined as
$$
L_j(\boldsymbol{g}) = \set{x}{\forall{k},~~ c(x,y_j) - g_j < c(x,y_k) - g_k}~,
$$
yielding a so-called Laguerre tessellation of $\mathcal{X}$. We can finally state the following important theorem, expressing the transportation maps in terms of the optimized Laguerre tessellation:
\begin{theorem}
The optimal transportation maps in the optimization problem in Eq.~\ref{eq: semi-discrete optimal transport} are given by the following formula:
\begin{equation}
    \hat{\gamma}(\de{x}|y_j) = \restr{\nu}{L_j(\hat{\boldsymbol{g}})}(\de{x})~,
\end{equation}
where $\restr{\nu}{L_j(\hat{\boldsymbol{g}})}$ denotes the renormalized restriction of the probability measure $\nu$ to the set $L_j(\hat{\boldsymbol{g}})$ and $\hat{\boldsymbol{g}}$ is the solution of the optimization in Eq.~\ref{eq: dual optimization}. \label{th 1}
\end{theorem}
\begin{proof}
The derivative of Eq.~\ref{eq: dual optimization II} with respect to $b_j$ is given by \cite{peyre2017computational}:
\begin{equation}
    \frac{\partial \mathcal{E}(\boldsymbol{g})}{\partial g_j} = - \int_{L_j(\boldsymbol{g})} \nu(\de{x}) + w_j~.
\end{equation}
Since the problem is unconstrained, this implies that, for the optimal dual weights $\hat{\boldsymbol{g}}$, we have 
\begin{equation}
    \int_{L_j(\hat{\boldsymbol{g}})} \nu(\de{x}) = w_j~.
\end{equation}
By plugging this result into Eq.~\ref{eq: dual optimization II}, we obtain:
\begin{align}
\OP{\mu}{\nu}{c} &= \sum_j \int_{L_j(\hat{\boldsymbol{g}})} c(x,y_j) \nu(\de{x}) \\
&= \sum_j w_j \int_{\mathcal{X}} c(x,y_j) \restr{\nu}{L_j(\hat{\boldsymbol{g}})}(x)~. 
\end{align}
By comparing this expression with the primal formulation in Eq.~\ref{eq: semi-discrete transport}, it follows immediately that the optimal transportation maps are given by the measures $\restr{\nu}{L_j(\hat{\boldsymbol{g}})}$.

\end{proof}

\section{Simultaneous optimization of weights and transportation maps}
In this section, we prove the main theoretical result behind the method. Consider the problem of finding the set of prototypes, weights and transportation maps that minimize the semi-discrete optimal transport divergence in Eq.~\ref{eq: semi-discrete optimal transport}. This results in the following joint optimization problem:
\begin{align}\label{eq: joint-optimization}
\arginf_{y_j, w_j, \gamma(x|y_j)}  \sum_j w_j \int_\mathcal{X} c(x,y_j) \gamma(\de{x}|y_j)~.
\end{align}
The solution of this optimization problem is an optimal ensemble of generative models. 

From the previous section, we know that the solution of the semi-discrete optimal transport problem is given by a tessellation of the target space into Laguerre cells $L_j(\boldsymbol{g})$ parameterized by a vector of dual weights $\boldsymbol{g}$. These cells are the support sets of the transportation maps. In the general case, these cells can be computed using computational geometry algorithms \cite{peyre2017computational}. Fortunately, the problem can be solved in closed form if we simultaneously optimize the weights and the transportation maps, as stated in the following theorem:

\begin{theorem}[Formal solution of the joint optimization problem] \label{th: formal solution}
The optimization problem given by
\begin{equation}
\arginf_{w_j, \gamma(\de{x}|y_j)} \sum_j w_j \int_\mathcal{X} c(x,y_j) \gamma(\de{x}|y_j)~,
\end{equation} \label{eq: minimax weights}
under the marginalization constraint given in Eq.~\ref{eq: marginalization constraint} is solved by the following Voronoi tessellation:
\begin{equation}
V_j = \set{x}{\forall{k},~~ c(x,y_j) < c(x,y_k)}~, \label{eq: voronoi set}
\end{equation}
where transportation maps are obtained by restricting the data distribution $\nu(\de{x})$ to each set of the tessellation
\begin{equation}
\hat{\gamma}(\de{x}|y_j) = \restr{\nu}{V_j}(\de{x}) \label{eq: minimax transporation maps}
\end{equation}
and the optimal weight are given by
\begin{equation}
\hat{w}_j = \int_{V_j} \nu(\de{x})~.
\end{equation}

\end{theorem}
 \begin{proof}
 It is easier to work with the dual optimization problem (see Eq.~\ref{eq: dual optimization II}). Enforcing the fact that the weight vector $\boldsymbol{w}$ should sum to one using Lagrange multipliers, we obtain the following unconstrained minimax optimization:
\begin{align}
&\inf_{\lambda} \inf_{\boldsymbol{w}} \sup_{\boldsymbol{g}} \mathcal{L}(\boldsymbol{w}, \boldsymbol{g}, \lambda) \\
&= \inf_{\lambda}\inf_{\boldsymbol{w}} \sup_{\boldsymbol{g}} \sum_j \int_{L_j(\boldsymbol{g})} \left( c(x,y_j) - g_j \right) \nu(\de{x}) + \sum_j g_j w_j + \lambda (1 - \sum_j w_j) ~. 
\end{align}
We can find the critical point by setting the gradient to zero:
\begin{align}
    &\frac{\partial \mathcal{L}}{\partial g_j} = - \int_{L_j(\boldsymbol{g})} \nu(\de{x}) + w_j = 0\\
    &\frac{\partial \mathcal{L}}{\partial w_j} = g_j - \lambda~. \\
\end{align}
The second equation implies that all the dual weights are equal to a constant. This implies that the Laguerre sets $L_j(\lambda)$ are Voranoi sets (Eq.~\ref{eq: voronoi set}). The first equation gives Eq.~\ref{eq: minimax weights} and the transportation maps in Eq.~\ref{eq: minimax transporation maps} are a consequence of Theorem~\ref{th 1}. Note that the resulting weights clearly respect the marginalization constraint.

\end{proof}
Using Theorem~\ref{th: formal solution}, we can write a simple expression for the optimal prototypes:
\begin{align}\label{eq: optimal prototypes}
\hat{y}_j = \arginf_{y_j} \int c(x,y_j) \restr{\nu}{V_j}(\de{x})~.
\end{align}
In other words, the prototypes are the medoid of the Voronoi sets with respect to the cost function $c(x, y)$.

\section{Learning the transportation maps}
The solution given in Eq.~\ref{eq: joint-optimization} is purely formal and does not directly provide a useful algorithm since the distribution that generated the data is not available. However, a practical algorithm can be obtained by minimizing a statistical divergence between this formal solution and a parametric model, such as a deep generative network. We will denote the probability measure induced by passing a latent measure $p(\de{z})$ though a deep network $F$ as ${F}_* p$ (where the bottom star denotes the pushforward of a measure through a measurable function). We approximate each optimal transportation map $\restr{\nu}{L_j}(\de{x})$ as follows:
$$
q_j(\de{x}) =  {F_j}_* p_j~.
$$
In practice, in the most naive implementation, this means that we reject samples that land outside $L_j$.
We train each network $F_j$ by minimizing a statistical divergence:
\begin{equation}\label{eq: divergence minimization}
\mathcal{L}_j = \divergence{\restr{\nu}{L_j}}{q_j}~.
\end{equation}
This is possible when the divergence $D$ solely requires the ability to sample from $\restr{\nu}{L_j}$ since we can sample from the dataset and reject all samples that do not land in $L_j$. For example, using the dual formulation of the Wasserstein distance, we can train the generators (Wasserstein GANs) by optimizing the following minimax problem using stochastic gradient descent (SGD) \cite{arjovsky2017wasserstein}:
\begin{equation}
\inf_F \sup_{f \in L^1} \W{\restr{\nu}{L_j}}{q_j}~ = \inf_F \sup_{f \in L^1} \left( \mean{f(x)}{q_j(\de{x})} - \mean{f(x)}{\restr{\nu}{L_j}} \right). \label{eq: wasserstein gan}
\end{equation}
where $L^1$ is the space of Lipschitz continuous functions. In practice, we approximate the samples $ $ with samples from a finite dataset and we parameterized both $F$ and $f$ as deep neural networks. Furthermore, $f$ is regularized by the following soft Lipschitz regularization term:
\begin{equation}\label{eq: regularization}
\mathcal{R}[F] = \gamma \mean{\text{ReLu}(|f(x) - f(y)| - 1} {x, y \sim_{\text{iid}} \restr{\nu}{L_j}}~.
\end{equation}
Using the trained generators, we can obtain a parameterized proxy for the loss in Eq.~\ref{eq: optimal prototypes} that we will use to train the prototypes:
\begin{align}\label{eq: optimal prototypes, proxy loss}
\mathcal{W}_j[y_j] = \int c(x,y_j) q_j(\de{x})~.
\end{align}

\section{The k-GANs algorithm} 
We are finally ready to formulate the algorithm. The basic idea is to minimize Eq.~\ref{eq: joint-optimization} using a two step approach similar to the expectation maximization scheme used in the k-means algorithm \cite{forgy1965cluster}. In the first step, we keep the generators fixed and we train the prototypes by minimizing Eq.~\ref{eq: optimal prototypes, proxy loss} with $n$ SGD steps. In the second step, we keep the prototypes (and consequently the tessellation) fixed and we train the generators by minimizing Eq.~\ref{eq: wasserstein gan} with $m$ SGD steps (cf. Algorithm~\ref{alg:kGANs}).

We named this algorithm k-GANs since it can be interpreted as a parametric version of the well-known k-medoids method. Specifically, a stochastic version of k-medoids is obtained if we replace the trained deep generators $F_j$ with  nonparametric generators $G_j$ that sample with uniform probability the elements of the training set that belongs to the $L_j$ set. This further reduces to a stochastic version of the k-means algorithm is we use the squared euclidean distance as cost function.

\begin{algorithm}
  \caption{k-GANs. k: Number of GANs, N: Number of epochs, M: Number of iterations per epoch. }\label{alg:kGANs}
  \begin{algorithmic}[1]
  \Procedure{k-GANs}{$k, N, M$}
    \State \texttt{Initialize k generators}
    \State \texttt{Initialize k discriminators}
    \State \texttt{Initialize the prototypes}
    \For{\texttt{n from 1 to N}} \Comment{loop over epochs}
      \For{\texttt{j from 1 to k}} \Comment{loop over GANs}
          \For{\texttt{m from 1 to M}} \Comment{loop over iterations}
          \State $\text{batch} \sim \text{Dataset}$ 
          \State $\text{batch}_j = [x ~ \text{for} ~x ~\text{in} ~ \text{batch} ~ \text{if} ~ x ~ \text{in} ~ V_j(\text{prototypes})]$ \Comment{reject outside the set $V_j$}
        \State \texttt{Train discriminator and generator using $\text{batch}_j$} \Comment{(Eq.~\ref{eq: wasserstein gan}, \ref{eq: regularization})}
        \State \texttt{Train prototypes using samples from the generator} \Comment{(Eq.~\ref{eq: optimal prototypes, proxy loss})}
        \EndFor
      \EndFor
    \EndFor
  \EndProcedure
  \end{algorithmic}
\end{algorithm}

\subsection{The k-generators algorithm} 
The theory outlined in this papar is not specific to GANs and can be directly applied to any generative model based on the minimization of a statistical divergence. For example, the approach can be used with variational autoencoders \cite{kingma2014auto} and sequential generative models such as those used in natural language processing \cite{sundermeyer2012lstm}.

\section{Choosing the cost function}
The clustering behavior of the k-GANs algorithm depends on the choice of the cost function $c(x,y)$. The shape of the cost determines the boundaries between the sets of the Voronoi tessellation. These boundaries are in general curved, except when $c(x,y)$ is a monotonic function of a quadratic form. 

\subsection{$L_p$. Euclidean and feature costs}
The simplest choice for the cost function is given by the $p$-th power of a $l_p$ norm:
\begin{equation}
    c_p(x, y) = \lpnorm{x - y}{p}^p~.
    \label{eq: lp norm}
\end{equation}
The boundaries induced by this family of norms are very well-studied and leads to different clustering behaviors \cite{hathaway2000generalized}. The most common choice is of course the familiar $L_2$ norm which to the familiar (Euclidean) k-means clusters. However, $l_p$ norms can lead to sub-optimal clustering in highly structured data such as in natural images as the boundaries tend to be driven by low-level features and ignore semantic information. A possible way of basing the partitioning on more semantic feature is to consider a $l_p$ norm in an appropriate feature space:
\begin{equation}
    c_p^{(f)}(x, y) = \lpnorm{f(x) - f(y)}{p}^p~,
    \label{eq: feature lp norm}
\end{equation}
where the feature map $f$ maps the raw data to a feature space. Usually, $f$ is chosen to been a deep network trained on a supervised task.

\subsection{Semi-supervised costs}
Another interesting way to insert semantic information into the cost function is to use labels on a subset of data. For example, we can have a cost of the following form:
\begin{equation}
    c_{ss}(x, y) = \theta(\mathfrak{l}(x), \mathfrak{l}(y)) c(_p(x,y) ~,
    \label{eq: semi-supervised}
\end{equation}
where the function $\mathfrak{l}(x)$ assign a discrete value based on whether the data-point $x$ is labeled and on its label. The function $\theta$ then scales the loss based on this label information. For example, $\theta$ can be equal to $0.1$ when two data-points have the same label, equal to $10$ when datapoints have different labels and equal to one when one or both of the datapoints are unlabeled. Note that, in order to use this semi-supervise cost in the k-GANs algorithm we need to be able to assign a label on the prototypes. A possibility is to train a classifier on the labeled part of the dataset. Alternatively, we can simply select the label of the closest labeled data-point. 


\section{Experiments}
In this section we validate the k-GANs method on a clustered toy dataset and on two image datasets: MNIST and fashion MNIST \cite{xiao2017fashion, lecun1998gradient}. We compare the performance of the k-GANs example against the performance of individual GANs. In all our experiments we used the Euclidean distance as cost function.

\begin{figure}[ht]
    \centering
    \includegraphics[width=0.7\textwidth]{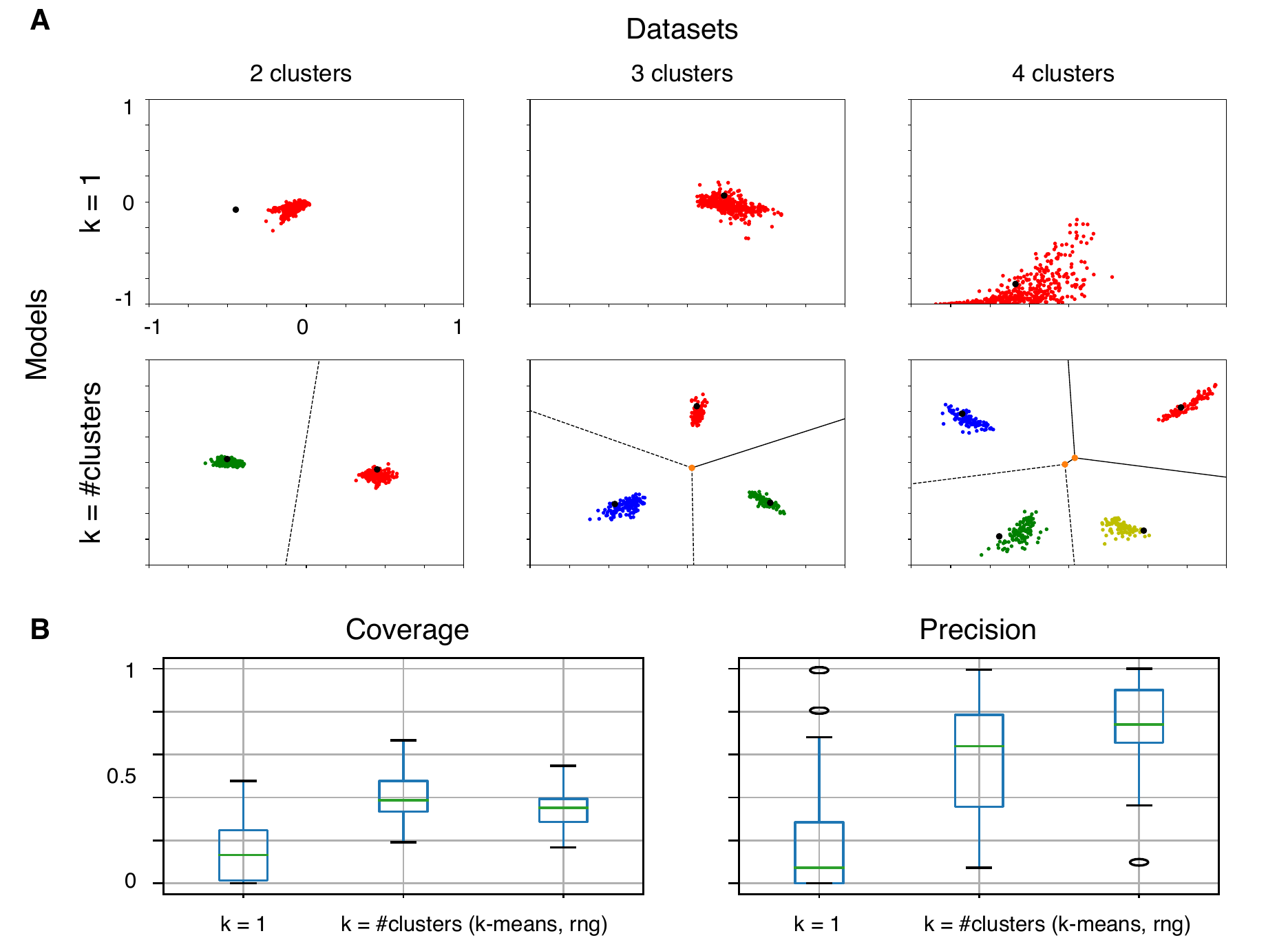}
    \caption{Results of the experiments in the toy dataset. A) Generated samples and Voronoi partition induced by the prototypes. The top row shows the result of the Wasserstein GAN baseline while the bottom shows the results for the k-GANs. B) Coverage and precision of the generated samples ensembled over the three toy datasets.}
    \label{figure 1}
\end{figure}

\subsection{Toy dataset}
We constructed several toy datasets (TD) comprised of randomly sampled coordinates on 2D plane, which were masked to create circular clusters. The first TD had two circular clusters of data points that fell within a radius of 0.25, centered on (-0.5, 0) and (0.5, 0); similarly the second TD had three circular clusters of data points, centered on (-0.5, -0.5), (0.5, -0.5) and (0, 0.5); and finally the third TD had four circular clusters centered on (-0.5, -0.5), (0.5, -0.5), (0.5, 0.5) and (-0.5, 0.5). We trained a Wasserstein k-GANs for $k$ ranging from 1 (baseline) to 4. We repeated the experiment 10 times. We used a 10-dimensional latent space for each of our generators. The generator network architecture was constructed as follows: a fully connected input layer of 32 units with batch normalization and leaky ReLU activation function, followed by fully connected layers of 16, 8 and finally 2 units. First two layers had batch normalization and leaky ReLUs, while the last one had sigmoid activation. The discriminator network had two fully connected layers with 16 and 8 units. For optimization, we used Adam with $\alpha = 10^{-4}$ for the generator and the discriminator networks, and $\alpha = 10^{-3}$ for the prototype. A burn-in parameter of 600 was introduced to the 60 000 iterations of training of each prototype and the corresponding generator/discriminator networks by minimizing Wasserstein distance. Figure~\ref{figure 1} shows the resulting tessellation and the sampled produced by each generator. The case corresponding to k equal to one is the baseline Wasserstein GAN. We evaluated the performance of the methods using two metrics: coverage and precision. Coverage is quantified by binning the plane in a 2D grid and counting the fraction of bins inside the circular masks that contain a generated data-point. The precision metric is given by the fraction of generated datapoints that are inside the masks. We compared a GAN baseline with two k-GANs runs where k was set equal to the number of clusters in the dataset. In one of these two runs the prototypes were initialized using k-means on the generated data while in the other they were sampled randomly from uniform distributions ranging from -1 to 1. Figure~\ref{figure 1} shows the metrics for all methods. Both k-GANs methods reach significantly higher performance than the baseline. 

\begin{figure}[ht]
    \centering
    \includegraphics[width=0.7\textwidth]{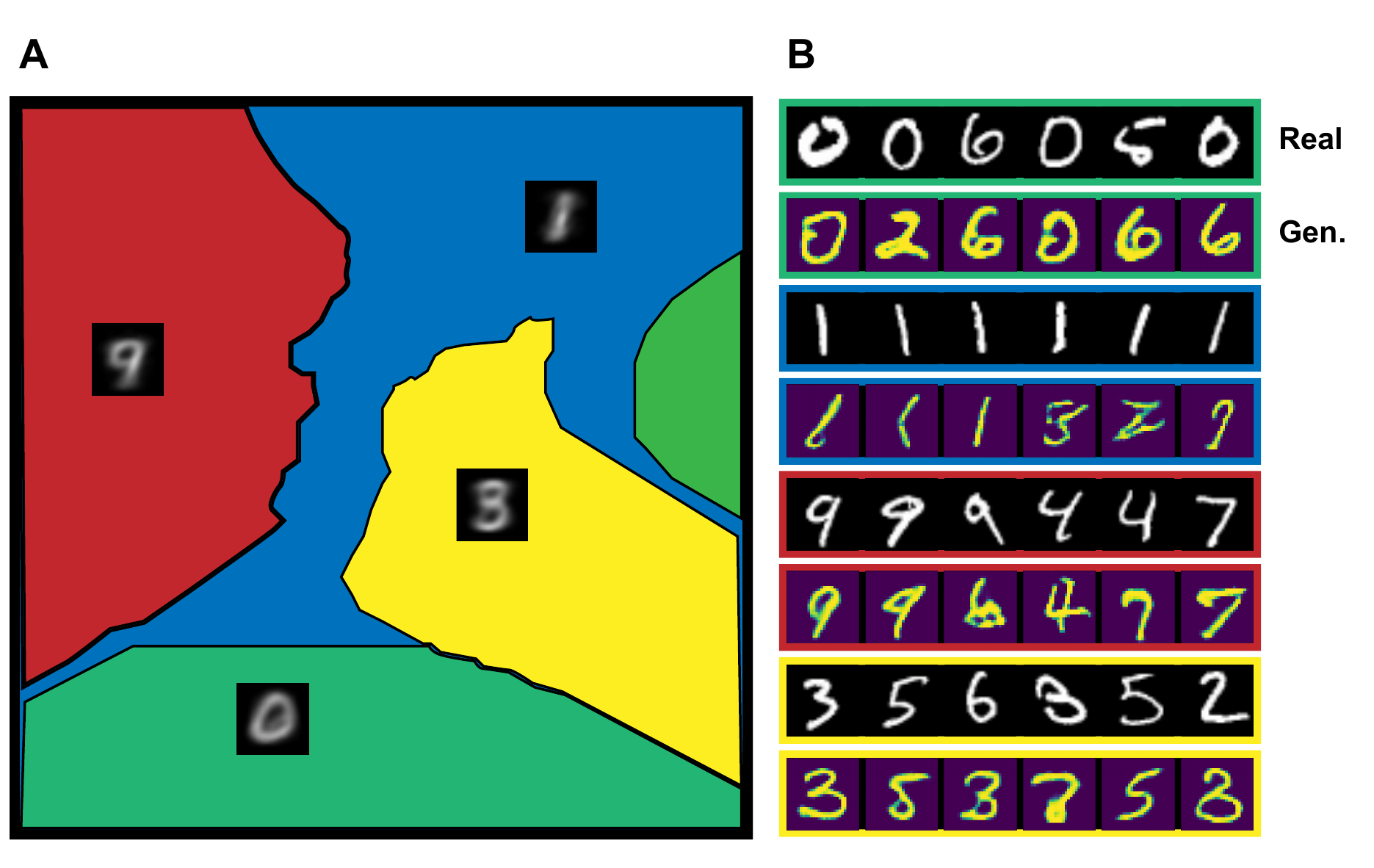}
    \caption{Results on MNIST with k = 4. A) Partition induced by the prototypes (black and white figures) in the t-SNE space. B) real (top row) and generated (bottom row) images corresponding to each prototype. The color surrounding the images matches the color scheme of the partition.}
    \label{figure 2}
\end{figure}

\subsection{MNIST and fashion MNIST}
We applied the k-GANs algorithm on MNIST and Fashion MNIST. We trained a Wasserstein k-GANs for $k$ ranging from 1 (baseline) to 4. Given our limited computational resources, we could train a single run on both models. Prototypes were initialized using k-means algorithm, and samples were assigned to the nearest prototype in batches of 100 during training. We used a 100-dimensional latent space for each of our generators. We used the following generator network architecture: a fully connected input layer of 12544 units with batch normalization and leaky ReLU activation function (output of which was reshaped to 256 x 7 x 7), followed by three deconvolution layers of 128, 64 and 1 units. First two had batch normalization and leaky ReLUs, while the last one had sigmoid activation. All of them had 5 x 5 kernels with a stride of 2 x 2 except for the first which had a stride of 1 x 1. The discriminator network had two convolutional layers with 64 and 128 units of size 5 x 5 and stride 2 x 2 and a linear layer of with a single output unit. Figure~\ref{figure 2} shows the results corresponding to $k = 4$. The figure shows the partition of the image space embedded into a 2D plane using t-SNE embedding. The images inside the sets are their prototypes. Figure~\ref{figure 2}B shows the real images and generated samples corresponding to each prototype. Figure \ref{figure 3} shows prototypes and samples on MNIST and fashion MNIST for the baseline and k-GANs with $k=4$. The k-GANs produced diversified samples except for one of the generator in fashion MNIST that collapsed on its mode. On the other hand, both the baseline models suffered from severe mode collapse. While it is difficult to draw strong conclusions from a single run, the results suggest that the k-GANs approach improves the stability of the base model.

\begin{figure}[h]
    \centering
    \includegraphics[width=0.7\textwidth]{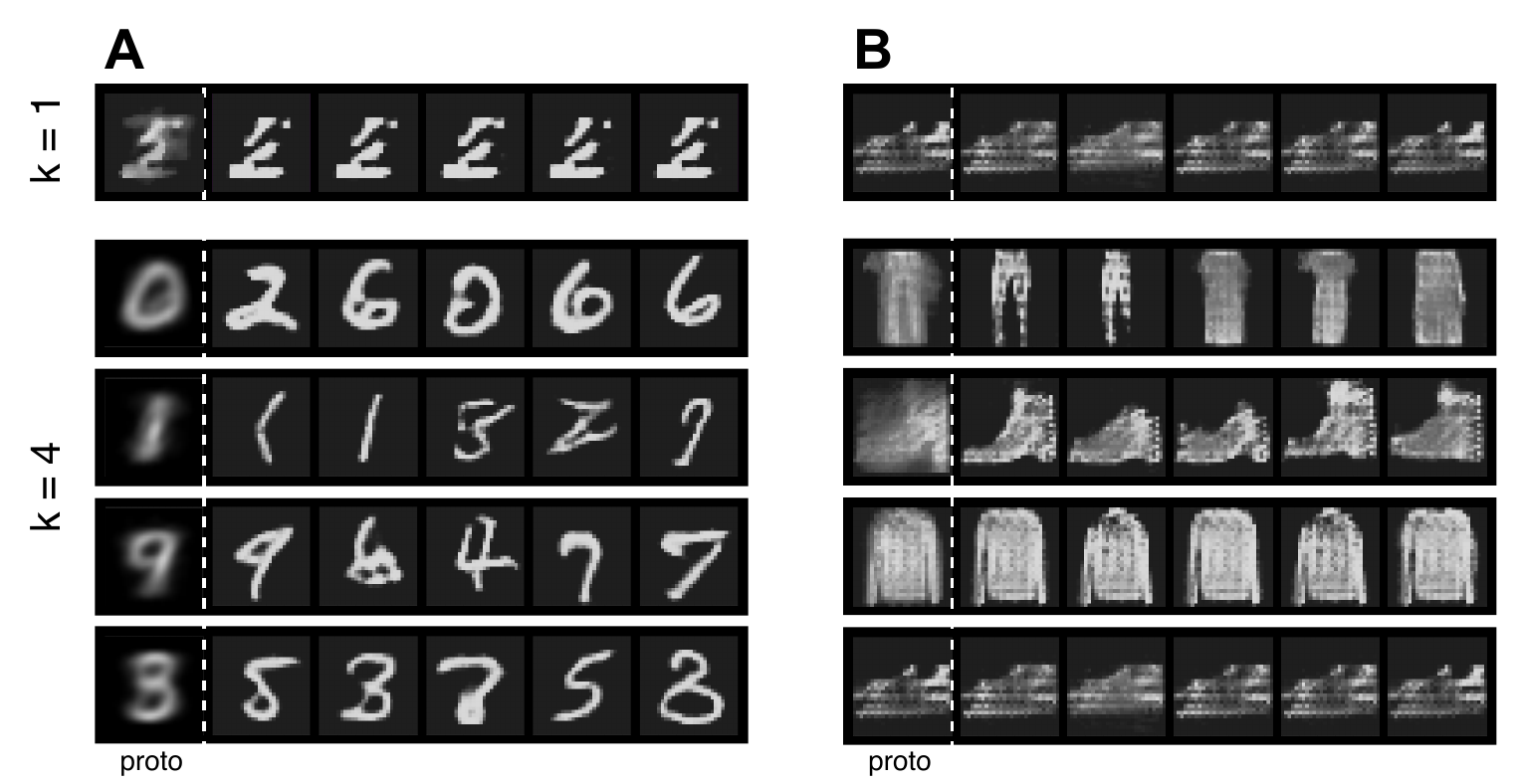}
    \caption{Samples of k-GANs and baselines for MNIST and fashion MNIST.}
    \label{figure 3}
\end{figure}

\section{Discussion}
In this paper we introduce method for training an ensemble of generators based on semi-discrete optimal transport theory. Each generator of the ensemble is associated to a prototype. These prototypes induce a partition of the data space and each set of the partition is modeled by an individual generator.  This protects the algorithm from mode collapse as each generator only needs to cover a localized portion of the data space.

\bibliographystyle{unsrtnat}
\bibliography{reference}

\end{document}